\documentclass[11pt,letterpaper]{article}

\usepackage{times}
\usepackage{mysetup,xspace,url}
\FULLPAGE

\usepackage{color-edits}
\addauthor{as}{red}      

\title{Dynamic ad allocation: bandits with budgets}
\author{Aleksandrs Slivkins%
\thanks{Microsoft Research Silicon Valley, Mountain View, CA 94043, USA. 
Email: {\tt slivkins@microsoft.com}. Parts of this work has been done while visiting Microsoft Research New York.}}
\date{May 2013}
\begin{document}

\maketitle

\renewcommand{\eqref}[1]{Equation~(\ref{#1})}

\newcommand{\problem}{\texttt{BudgetedAdsMAB}\xspace}
\newcommand{\UCB}{\ensuremath{\mathtt{UCB1}}\xspace}
\newcommand{\myUCB}{\ensuremath{\mathtt{BudgetedUCB}}\xspace}

\newcommand{\regret}{\mathtt{Regret}}
\newcommand{\reward}{\mathtt{Reward}}

\newcommand{\campaigns}{\mathcal{C}}

\newcommand{\kA}{i_\mathtt{A}} 
\newcommand{\kB}{i_\mathtt{B}} 
\newcommand{\kBub}{\mathtt{UCB}(i_\mathtt{B})} 

\begin{abstract}
We consider an application of multi-armed bandits to internet advertising (specifically, to dynamic ad allocation in the pay-per-click model, with uncertainty on the click probabilities). We focus on an important practical issue that advertisers are constrained in how much money they can spend on their ad campaigns. This issue has not been considered in the prior work on bandit-based approaches for ad allocation, to the best of our knowledge.

We define a simple, stylized model where an algorithm picks one ad to display in each round, and each ad has a \emph{budget}: the maximal amount of money that can be spent on this ad. This model admits a natural variant of UCB1, a well-known algorithm for multi-armed bandits with stochastic rewards. We derive strong provable guarantees for this algorithm.
\end{abstract}

\section{Introduction}
\label{sec:intro}

Multi-armed bandits (henceforth, \emph{MAB}), and more generally online decision problems with partial feedback and exploration-exploitation tradeoff, has been studied since 1930's in Operations Research, Economics and several branches of Computer Science \cite{Sundaram-survey03,CesaBL-book,Gittins-book11,Bubeck-survey12}.
Such problems arise in diverse domains, e.g., the design of medical experiments, dynamic pricing, and routing in the internet. In the past decade, a surge of interest in MAB problems has been due to their applications in web search and internet advertising.

In the most basic MAB problem \cite{bandits-ucb1}, an algorithm repeatedly chooses among several possible actions (traditionally called \emph{arms}), and observes the reward for the chosen arm. The rewards are \emph{stochastic}: the reward from choosing a given arm is an independent sample from some distribution that depends on the arm but not on the round in which this arm is chosen. These reward distributions are not revealed to the algorithm. The algorithm's goal is to maximize the total expected reward over the time horizon.

This paper is concerned with an application of MAB to Internet advertising. This application considers advertisers that derive value when users click on their ads. A predominant market design for such advertisers is \emph{pay-per-click}: advertisers pay only when their ads are clicked. Users arrive over time, and an algorithm needs to choose which ads to show to each user. Both the ad market and the advertisers experience significant uncertainty on click probabilities;%
\footnote{Click probabilities are also called \emph{click-through rates} in the industry, or \emph{CTRs} for short.}
the estimates of CTRs can be refined over time. It is because of this uncertainty on CTRs that MAB are relevant to this application domain.

A standard, and very stylized, way to model these ad-related issues in the MAB framework is as follows (e.g., see~\cite{yahoo-bandits07}). An algorithm chooses one ad in each round (so ads correspond to arms in MAB), and observes whether this ad is clicked on. For each click on every ad $i$, algorithm receives a fixed payment $b_i$ from the corresponding advertiser. Thus, the expected reward from showing ad $i$ is equal to $b_i$ times the CTR for this ad. The CTRs are not initially known to the algorithm. The algorithm's goal is to maximize the total expected reward.

To the best of our knowledge, prior work on MAB-based approaches to ad allocation has ignored an important practical issue: advertisers are constrained in how much money they can spend on their ad campaign. In particular, each advertiser typically has a \emph{budget}: the maximal amount of money she is allowed to spend. This is the issue that we focus on in this paper.

\subsection{Problem formulation: \problem}

There are $k$ advertisers (arms), each with one ad that she wishes to be displayed. Each ad $i$ is characterized by the following three quantities: CTR $\mu_i\in[0,1]$, \emph{payment-per-click} $b_i$ and \emph{budget} $B_i$. The payments-per-click and the budgets are revealed to the algorithm, but the CTRs are not.

In each round an algorithm picks one ad. This ad is displayed (receives an \emph{impression}), and the algorithm observes whether this ad is clicked on. The click on a given ad $i$ happens independently (from everything else), with probability $\mu_i$. If ad $i$ is clicked, the corresponding advertiser is charged $b_i$, and her remaining budget is decreased by this amount. An arm is \emph{available} (can be chosen) in a given round only if its remaining budget is above $b_i$. There is a time horizon $T$. The goal of the algorithm is to maximize its expected total reward, where the total reward is the sum of all charges.

This is a non-Bayesian (prior-independent) formulation: there are no priors on the CTRs that are available to the algorithm, and we are looking for guarantees that hold for any prior.

The expected value of one impression of arm $i$ is $w_i \triangleq b_i\,\mu_i$. For ease of exposition, we re-order the arms so that
    $w_1 \geq w_2 \geq \ldots \geq w_k$.

\xhdr{Benchmark.}
We use the \emph{omniscient benchmark}, standard benchmark in the literature on MAB and related problems. This is the best algorithm that knows all latent information in the problem instance (in this case, the CTRs). In this problem, the omniscient benchmark is very simple: play arm $1$ while it is available, then play arm $2$ while it is available, and so on. Call it the \emph{greedy benchmark}. Performance of an algorithm $\mA$ is measured as \emph{greedy regret} (regret with respect to the greedy benchmark), defined as expected reward of the greedy benchmark minus the expected reward of the algorithm. Denote it $\regret(\mA)$.

It is worth noting that, given the optimality of the greedy benchmark, the \emph{best fixed arm} -- another standard benchmark in the literature on MAB -- is not informative for our setting.

\subsection{Our contributions}

We consider a natural algorithm and prove that it works quite well. While the algorithm is essentially the first thing a researcher familiar with prior work on MAB would suggest, our technical contribution is the analysis of this algorithm, and particularly the ``coupling argument" therein. \OMIT{The analysis builds on several ideas from prior work on MAB; its crux is a new coupling argument.} The conceptual contribution is that we provide an assurance that the natural approach works, from a theoretical point of view, and suggest the strengths and limitations of this approach.

Our algorithm, called \myUCB, is a natural modification of \UCB~\cite{bandits-ucb1}, a well-known algorithm for MAB with stochastic rewards. \UCB maintains a numerical score (\emph{index}) for each arm, and in every round chooses an arm with the largest index. The index of arm $i$ is, essentially, the best available upper confidence bound on the expected reward from this arm. \myUCB chooses, in each round, an arm with the maximal index \emph{among all available arms}. (So the two algorithms coincide if the budgets are infinite.)

We formulate our provable guarantees in terms of the last arm whose budget is exhausted by the greedy benchmark. (Recall that the arms $i$ are ordered in the order of decreasing $w_i = b_i\, \mu_i$.) Denote this last arm $\kB$ if it exists; set $\kB=0$ otherwise.  Since $\kB$ is a random variable, the regret bound is in expectation over the randomness in $\kB$. For most problem instances $\kB$ is highly concentrated: it is typically within $\pm 1$ from its expectation.

\begin{theorem}\label{thm:main}
Consider \problem. For each $\eps>0$ it holds that
\begin{align}\label{eq:thm-main}
\regret(\myUCB) \leq \eps T + O(\log T)\;
            \E\left[ \max_{i\in \{\kB,\,\kB+1\}}\;
                    \sum_{j=i+1}^k\; \frac{b_j^2}{\max(\eps,w_i-w_j)}
\right],
\end{align}
where the expectation is over the randomness in $\kB$.
\end{theorem}

The regret bound~\refeq{eq:thm-main} is driven by the differences
    $\Delta(i) = w_i - w_{i+1}$,
more specifically by the random quantity $\Delta(\kB)$. We derive a ``pessimistic" corollary for the case when $\Delta(\kB)$ may be arbitrarily small, and an ``optimistic" corollary for the case of large $\Delta(\kB)$.~
\footnote{To derive Corollary~\ref{cor:main}, we pick $\eps = \sqrt{\tfrac{\log T}{kT}}$ in~\eqref{eq:thm-main} for part (a), and
    $\eps = kv^2/T^2$ for part (b).}

\begin{corollary}\label{cor:main}
Consider \problem. Denote
    $v^2 = \tfrac{1}{k}  \sum_{j=1}^k b_j^2$.
\begin{itemize}
\item[(a)] $ \regret(\myUCB) \leq  O(v\, \sqrt{kT\log T})$.

\item[(b)]
$\regret(\myUCB) \leq
    O\left( \tfrac{k}{\delta}\, v^2\, \log T \right)$
for any $\delta>0$ such that
   $ \Pr[ \Delta(\kB) \geq \delta] \geq 1- (v/T)^2 $.
\end{itemize}
\end{corollary}

The regret bounds in this corollary extend the corresponding ``pessimistic" and ``optimistic" guarantees for \UCB from the special case of MAB with stochastic rewards (i.e., no budgets and $b_j\equiv 1$) to the full generality of \problem.%
\footnote{Without budgets, we have $\kB=1$ and therefore the assumption in Corollary~\ref{cor:main}(b) reduces to $\Delta(1)\geq \delta$.}
Both guarantees are nearly optimal for this special case, respectively up to $O(\log T)$ factors and up to constant factors  \cite{Lai-Robbins-85,bandits-ucb1,bandits-exp3}.

Interestingly, all above regret bounds do not depend on the budgets.

\subsection{Discussion}

One common criticism of the work on non-Bayesian (prior-independent, regret-minimizing) MAB problems is that the algorithmic ideas and proof techniques introduced for the numerous MAB models studied in the literature are too specific to their respective models, and do not easily generalize to more general settings that are common in applications. In view of this criticism, it is useful to identify general ideas and techniques that one can build on when working on the (more) general settings, and provide concrete examples of \emph{how} one can build on them. The present paper contributes to this direction: we build on the algorithmic idea of ``UCB indices", and a certain proof technique to analyze them (both from~\cite{bandits-ucb1}). These ideas have been tremendously useful in several other MAB settings with stochastic rewards e.g.~\cite{sleeping-colt08,Munos-nips08,Bubeck-colt10,contextualMAB-colt11,Csaba-nips11}.

It is worth noting that \myUCB does not need to input the budgets: instead, it can be implemented via an oracle that determines whether a given arm is available in a given round. In other words, advertisers do not need to submit their budgets upfront; instead, they only need to notify the algorithm whether they are still willing to participate in a given round. This is useful because an advertiser may be reluctant to commit to a specific budget and/or reveal it early in her ad campaign. Also, she may choose to strategically misreport the budget if asked.

\section{Our algorithm: \myUCB}
\label{sec:alg}

Our algorithm, called $\myUCB$, is a natural extension of the well-known algorithm \UCB~\cite{bandits-ucb1}.

For each arm $i$ and time $t$, let $c_i(t)$ and $n_i(t)$ be, respectively, the number of clicks and the number of impressions of this arm up to (but not including) time $t$. Define the \emph{confidence radius} of arm $i$ as
\begin{align}\label{eq:conf-rad-defn}
 r_i(t) \triangleq C\, \sqrt{\frac{\log T}{1+n_i(t)}} .
\end{align}
Here $C$ is some constant to be chosen later. Informally, the meaning of $r_i(t)$ is that
\begin{align}\label{eq:conf-rad-meaning}
|\mu_i(t) - \nu_i(t)| \leq r_i(t)
\end{align}
holds with high probability, where
    $\nu_i(t)\triangleq c_i(t)/n_i(t)$
is the (current) average CTR.

Define the \emph{UCB index} of arm $i$ as
\[ I_i(t) \triangleq b_i(\nu_i(t) + r_i(t)). \]
Note that the index of arm $i$ is an upper confidence bound (\emph{UCB}) on the quantity $b_i\mu_i$ which represents the expected value of one impression of $i$.

Now that the index is defined, the algorithm is very simple: among available arms, pick an arm with the maximal index, breaking ties arbitrarily.

\xhdr{Discussion.}
The original algorithm \UCB in~\cite{bandits-ucb1} is, essentially, a special case of \myUCB when all arms are available and all values are $b_i=1$. Moreover, the algorithm in~\cite{sleeping-colt08} for sleeping bandits with stochastic rewards coincides with ours for $b_i\equiv 1$ (but the analysis from~\cite{sleeping-colt08} does not carry over to our setting, see Section~\ref{sec:related-work} for more discussion).

Most likely, the $\log T$ in the definition of the confidence radius can be replaced by $\log t$, which should lead to improved constant factors in the regret bounds. In particular, the algorithms in~\cite{bandits-ucb1} and~ \cite{sleeping-colt08} have $\log t$ there. We use $\log T$ because it makes our analysis easier, and increases the regret by at most a constant factor.

\section{Analysis: proof of Theorem~\ref{thm:main}}

The technical contribution of this paper is the analysis of \myUCB. The crux thereof is the ``coupling argument" encapsulated in Lemma~\ref{lm:analysis-coupling}. To argue about random clicks, an important conceptual step is to consider two different representations of realized clicks (defined below). Also, we build on the technique from the analysis of \UCB~\cite{bandits-ucb1}, which is encapsulated in Lemma~\ref{lm:UCB-trick}.


\xhdr{Notation.}
Consider an execution of $\myUCB$. For each arm $i$, let $n_i(t)$ be the number of impressions of arm $i$ before round $t$. Let $n_i = n_i(T+1)$ be the total number of impressions from arm $i$. Let $\vec{n} = (n_1,\ldots,n_k)$ be the \emph{impressions vector} for \myUCB. Similarly, let $\vec{m}$ be the impressions vector for the greedy benchmark. Note that $\vec{n}$ and $\vec{m}$ are random variables. Let
    $\vec{w} = (w_1,\ldots,w_k)$,
where $w_i = b_i\, \mu_i$.

\xhdr{Click realizations.}
We will use two ways to represent the realization of the random clicks. Each representation is a 0-1 matrix, denoted $Y = (Y_{i,t})$ and $Y' = (Y'_{i,t})$ respectively, where rows $i$ range over ads and columns $t$ range over rounds. The first representation, called \emph{per-round realization}, is as follows: if arm $i$ is selected in round $t$ then it is clicked if and only if $Y_{i,t}=1$. The second realization, called the \emph{stack realization}, is as follows: the $t$-th time arm $i$ is selected, it is clicked if and only if $Y'_{i,t}=1$. Note that for each pair $(i,t)$, both $Y_{i,t}$ and $Y'_{i,t}$ are independent 0-1 random variables with expectation $\mu_i$.

While each of the two representations suffices to formally represent the random clicks, we find it convenient to use both. In particular, the per-round realization is used in Claim~\ref{cl:reward-simple}, and the stack realization is used in Claim~\ref{cl:conf-rad} and in the coupling argument in Lemma~\ref{lm:analysis-coupling}.

\begin{claim}\label{cl:reward-simple}
$\E[\reward(\myUCB)] = \E[ \vec{n}\cdot\vec{w} ]$.
\end{claim}

\begin{proof}
Let $X_{it}\in \{0,1\}$ be $1$ if and only if arm $i$ is selected in round $t$. Let $\{Y_{i,t}\}$ be the per-round realization. Since for each pair $(i,t)$ the random variables $X_{i,t}$ and $Y_{i,t}$ are mutually independent, it follows that
\[ \E[X_{i,t}\, Y_{i,t}] = \E[X_{i,t}]\, \E[Y_{i,t}] = \mu_i\, \E[X_{i,t}]. \]
Noting that
    $\reward(\myUCB) = \sum_{i,t} b_i\, X_{i,t} Y_{i,t}$,
we have
\begin{align*}
\E[\reward(\myUCB)]
    &= \textstyle \sum_{i,t} b_i\; \E[ X_{i,t} Y_{i,t} ] \\
    &= \textstyle \sum_{i,t} b_i\,\mu_i\; \E[ X_{i,t}] \\
    &= \textstyle \sum_i b_i\,\mu_i \; \E[\sum_t X_{i,t}] \\
    &= \textstyle \sum_i w_i\, \E[n_i]. \qedhere
\end{align*}
\end{proof}

Similarly, expected reward of the greedy benchmark is $\E[ \vec{n}\cdot \vec{w} ]$.

\begin{corollary}\label{cor:regreg}
$\regret(\myUCB) = \E[ (\vec{m}-\vec{n})\,\cdot\vec{w} ]$.
\end{corollary}

We pick the constant $C$ in~\eqref{eq:conf-rad-defn} so that~\eqref{eq:conf-rad-meaning} holds with
 \emph{really} high probability, so that the failure event when~\eqref{eq:conf-rad-meaning} does not hold can, essentially, be ignored in the analysis.%
\footnote{While $C=10$ suffices for the analysis, prior work on \UCB-style algorithms (e.g. in~\cite{RBA-icml08,ZoomingRBA-icml10}) suggests that a smaller value such as $C=1$ can be used in practice.}

\begin{claim}\label{cl:conf-rad}
With probability at least $1-\tfrac{1}{T}$, for each arm $i$ and each time $t$~\eqref{eq:conf-rad-meaning} holds.
\end{claim}

\begin{proof}[Proof Sketch]
Consider the stack realization  $(Y'_{i,t})$. For each arm $i$ and each time $t$, apply Chernoff Bounds to the sum $\sum_{s=1}^t Y'_{i,t}$ (which is the number of clicks in the first $t$ times that arm $i$ is selected). Then take the Union Bound over all $i$ and all $t$.
\end{proof}

In the rest of the proof we will assume without further notice that the event~\eqref{eq:conf-rad-meaning} holds for each arm $i$ and each time $t$. Essentially, we will argue deterministically from now on, whereas all ``probabilistic" reasoning is contained in Claim~\ref{cl:reward-simple} and Claim~\ref{cl:conf-rad}.

The following lemma says that each sub-optimal arm is not played too often. This is the crucial part of a UCB-style analysis, and it incorporates the main trick from the original analysis in~\cite{bandits-ucb1}.

\begin{lemma}\label{lm:UCB-trick}
Let $i^*_j$ be the best (lowest numbered) available arm at the last time when arm $j$ has been selected.
Then for each arm $j$ such that $j\neq i^*_j$ it holds that
\begin{align}\label{eq:UCB-trick}
 n_j \leq O(\log T)\; \left( \frac{b_j}{w(i^*_j)-w(j)}\right)^2 .
\end{align}
\end{lemma}
\begin{proof}
We will use the fact that by~\eqref{eq:conf-rad-meaning} for each arm $j$ and each arm $t$ it holds that
\[ w_j \leq  I_j(t) \leq w_j + 2\, b_j\, r_j(t). \]

Let $t$ be the last round when arm $i$ has been selected, and denote $i = i^*_j$. Since arm $i$ has been selected in round $t$, it must have had the highest index at the time. Therefore
\[ w_i \leq I_i(t) \leq I_j(t) \leq w_j + 2\,b_j\,r_j(t).\]
It follows that
    $w_i-w_j \leq 2\, b_j\,r_j(t) = O(b_j)\sqrt{\frac{\log T}{n_j}}$,
which implies the desired bound~\refeq{eq:UCB-trick}.
\end{proof}

From now on assume that \myUCB and the greedy benchmark are run on the same stack realization. Arguments in which two random processes are run on a joint probability distribution (\emph{coupled}) with the same marginal distributions for each process are known in Probability Theory as \emph{coupling arguments}.

We encapsulate the coupling argument in the following lemma. To state this lemma, recall that $\kB$ is the last (highest-numbered) arm exhausted by the greedy benchmark if such arm exists, and $0$ otherwise. Let $\kA$ be the best (lowest-numbered) arm that is \emph{not} exhausted by \myUCB.

\begin{lemma}\label{lm:analysis-coupling}
$(\vec{m}-\vec{n})\cdot \vec{w} \leq \sum_{j= \max(\kA,\kB)+1}^k\; n_j(w_{\kA} - w_j)$
    where $\kA\leq \kB+1$.
\end{lemma}
\begin{proof}
We consider three cases. The first case is when no arms are exhausted by the greedy benchmark. Then $\kB=0$, and the greedy benchmark played arm $1$ for $T$ rounds, so $m_1=T$ and $m_j=0$ for all $j\geq 2$. Therefore:
\[ (\vec{m}-\vec{n})\cdot \vec{w}
    = \textstyle (T-n_1)w_1 - \sum_{j=2}^k n_j\, w_j
    = \textstyle \sum_{j=2}^k n_j (w_1-w_j). \]
Moreover, since the greedy benchmark has not exhausted arm $1$, \myUCB has not exhausted it either, so $\kA=1$ and we are done.

For the other two cases let us assume that the greedy benchmark exhausts at least one arm (i.e., $\kB\geq 1$). We claim that for each arm $i\leq \kB$ it holds that $n_i\leq m_i$. Indeed, the greedy benchmark exhausts each arm $i\leq \kB$, and, since \myUCB and the greedy benchmark use the same stack realization, \myUCB would also exhaust arm $i$ after $n_i$ impressions, after which this arm would not be available. Claim proved.

The second case is that $\kB\geq 1$ and $n_j = m_j$ for each arm $j\leq \kB $. Then $\kA = \kB+1$. (Indeed, if \myUCB exhausted arm $\kB+1$ then the greedy benchmark would have also exhausted it, contradiction.) Let $i=\kA$ and note that $n_i\leq m_i$. It follows that
\begin{align*}
 (\vec{m}-\vec{n})\cdot \vec{w}
    &= \textstyle \sum_{j\geq i}\; (m_j - n_j) w_j \\
    &= \textstyle (m_i-n_i) w_i - \sum_{j\geq i+1}\; n_j\, w_j \\
    &= \textstyle \sum_{j\geq i+1}\; n_j\, (w_i-w_j).
\end{align*}

The remaining third case is that $\kB\geq 1$ and $n_j < m_j$ for some arm $j\leq \kB $. Then $\kA$ is the lowest-numbered such arm; in particular, $\kA\leq \kB$. Let $i=\kA$ and $\ell = \kB+1$. Note that we do not know whether $n_\ell\leq m_\ell$, and so we have to allow for the possibility that $n_\ell> m_\ell$. Then:
\begin{align*}
\textstyle \sum_{j\leq \kB}\; (m_j-n_j)w_j
    &\leq \textstyle m\,w_i \; \text{ where }
         m\triangleq \sum_{j\leq \kB}\; (m_j-n_j)w_j. \\
(\vec{m}-\vec{n})\cdot \vec{w}
    & \textstyle \leq m\, w_i - (n_\ell-m_\ell)w_\ell -  \sum_{j\geq \ell+1} n_j w_j \\
    & \textstyle = \sum_{j\geq \ell+1} n_j (w_i-w_j) + (n_\ell-m_\ell)(w_i-w_\ell) \\
    & \textstyle = \sum_{j\geq \ell} n_j (w_i-w_j).
\end{align*}
This completes the third case.

In all three cases we regroup the terms in the sums using the fact that
    $\sum_i n_i = \sum_i m_i = T$.
\end{proof}

Let $i=\max(\kA,\kB)$ and let
    $S = \{ j>i:\, w_{\kA}-w_j \geq \eps\}$.
Then
\[ \textstyle \sum_{j=i+1}^k\; n_j (w_{\kA}-w_j)
    \leq \eps T +  \sum_{j\in S}\; n_j (w_{\kA}-w_j). \]
By Lemma~\ref{lm:UCB-trick}, noting that $i^*_j \leq \kA$, we have for each $j>i$ that
\[ n_j \leq \frac{O(b_j^2\, \log T)}{(w(i^*_j)-w(j))^2}
    \leq \frac{O(b_j^2\, \log T)}{(w_{\kA}-w_j)^2}. \]
Putting it all together, we obtain the following:
\begin{align} \label{eq:analysis-best}
(\vec{m}-\vec{n})\cdot \vec{w}
    &\leq \eps T + \sum_{j\in S}\; \frac{O(b_j^2\,\log T)}{w_{\kA}-w_j}.
\end{align}
For Theorem~\ref{thm:main} we use a somewhat weaker corollary of~\eqref{eq:analysis-best} which gets rid of $\kA$.
\begin{align} \label{eq:analysis-output}
(\vec{m}-\vec{n})\cdot \vec{w}
    &\leq \eps T + \max_{i\in \{\kB,\,\kB+1\}}\; \sum_{j=i+1}^k\; \frac{O(b_j^2\,\log T)}{\max(\eps,w_i-w_j)}.
\end{align}
Using Corollary~\ref{cor:regreg} and taking expectations in both sides of \eqref{eq:analysis-output}, we obtain the desired regret bound~\refeq{eq:thm-main} in Theorem~\ref{thm:main}.

\section{Related work}
\label{sec:related-work}

MAB has been an active area of investigation since 1933~\cite{Thompson-1933}, in Operations Research, Economics and several branches of Computer Science: machine learning, theoretical computer science, AI, and algorithmic economics. A survey of prior work on MAB is beyond the scope of this paper; a reader is encouraged to refer to \cite{CesaBL-book,Bubeck-survey12} for background on prior-independent MAB, and to \cite{Sundaram-survey03,Gittins-book11} for background on Bayesian MAB. Starting from \cite{yahoo-bandits07}, much of the work on MAB has been motivated by internet advertising. Below we only discuss the work directly relevant to this paper.

The present paper continues the line of work on prior-independent MAB with stochastic rewards (where the reward of a given arm $i$ is an i.i.d. sample of some time-invariant distribution). The basic formulation for MAB with stochastic rewards is well-understood
(\cite{Lai-Robbins-85,bandits-ucb1} and the follow-up work, see \cite{Bubeck-survey12} for references and discussion).

\OMIT{
The ideas from the algorithm and analysis of \UCB have been tremendously useful in more general settings of MAB with stochastic rewards e.g.~\cite{DynamicMAB-colt08,sleeping-colt08,Munos-nips08,Bubeck-colt10,contextualMAB-colt11,Csaba-nips11,DynPricing-ec12}.}

Our formulation is a special case of \emph{sleeping bandits}~\cite{sleeping-colt08,contextualMAB-colt11} where in each round, a subset of arms is not available (``asleep") and the goal is to compete with the best available arm. Available arms for a given round are chosen by an adversary. However, this adversary in~\cite{sleeping-colt08,contextualMAB-colt11} is \emph{oblivious} (it decides its selections for all rounds before round $1$), whereas in our problem it is \emph{adaptive} (it decides its selection for round $t$ only after observing what happened before). This is a significant complication. To the best of our knowledge, the results in~\cite{sleeping-colt08,contextualMAB-colt11} do not extend to settings where available arms are chosen by an adaptive adversary.

Sleeping bandits are in turn a special case of \emph{contextual bandits}, where in each round an oblivious adversary provides a \emph{context} $x$ which determines which arms are available and, moreover, what are the expected payoffs in this round. The goal is to compete with the best (available) arm for a given context. Contextual bandits have been a subject of much recent work, see \cite{Bubeck-survey12} for a survey.

\OMIT{The connection between sleeping bandits and contextual bandits is made very explicit in~\cite{contextualMAB-colt11}.}

Several recent papers consider MAB problems with a single limited resource that is consumed by the arms. In such problems, each round yields a reward and a resource consumption, both of which may (stochastically) depend on the chosen arm. A typical example is ``dynamic selling"\cite{BZ09,DynPricing-ec12}, where a seller has a limited supply of items and offers one item for sale in each round; the arms correspond to the offered prices. Other examples include ``dynamic buying"~\cite{DynProcurement-ec12} (where a buyer has a limited budget of money and interacts with a new seller in each round), and several versions in which the resource consumption for a given arm is deterministic \cite{GuhaM-icalp09,GuptaKMR-focs11,TranThanh-aaai10,TranThanh-aaai12}. To the best of our knowledge, no published prior work has addressed MAB with multiple resources / budgets.

A very recent, yet unpublished, paper~\cite{BwK-full}, concurrent with respect to this paper, considers a generalization of our setting in which the budgets can be specified for arbitrary subsets of ads. They design new algorithms, based on techniques that are very different from ours. (Their algorithms and their analysis extend to a very general setting of MAB with arbitrary knapsack-style constraints, for which ad allocation is one of the application domains.) However, the guarantees in~\cite{BwK-full} for \problem are much weaker than ours. Essentially, they obtain regret
    $O(\sqrt{kT}\;(1+\sqrt{T/B}))$,
where $B$ is the smallest budget; this is not a very strong  guarantee if $B$ is small. Moreover, their analysis does not imply an ``optimistic" corollary similar to Corollary~\ref{cor:main}(b).

\xhdr{Ad allocation.} A large amount of work has addressed ad allocation in the internet settings. Most papers in this area do not consider the issue of uncertainty on the CTRs. Some of the prominent themes is online matching (of ads and webpages) and the design of \emph{ad auctions} (where the key issue is that the advertisers may strategically manipulate their bids if it benefits them). A more detailed discussion of this work is beyond the scope of this paper; see Chapter 28 of \cite{NRTV07} for background.

In the literature on ad auctions, most relevant to our work are the papers that address the strategic issues jointly with the issue of uncertainty on CTRs and/or advertisers' values-per-click (if these values change over time). There are two somewhat distinct directions:
    \emph{dynamic auctions}, in which the advertisers submit bids over time (see~\cite{DynAuctions-survey11} for a survey), and
    \emph{MAB mechanisms}~\cite{MechMAB-ec09,DevanurK09,Transform-ec10,Gatti-ec12},
where the advertisers submit bids only once, and the mechanism allocates ads over time.

\section*{Acknowledgements}

The author would like to thank Ashwin Badanidiyuru, Sebastien Bubeck and Robert Kleinberg for many stimulating conversations about multi-armed bandits.


\begin{small}
\bibliographystyle{abbrv}
\bibliography{bib-abbrv,bib-bandits,bib-slivkins,bib-AGT}
\end{small}

\end{document}